\documentclass{llncs}
\usepackage{amsmath}
\usepackage{amssymb}
\usepackage{times}
\usepackage{indentfirst}
\usepackage{graphicx}
\usepackage{fancyhdr}
\usepackage{epstopdf}%eps文件使用
\begin{document}
\title{Rough matroids based on coverings}

\author{Bin Yang, Hong Zhao and William Zhu~~\thanks{Corresponding author.
E-mail: williamfengzhu@gmail.com (William Zhu)}}
\institute{
Lab of Granular Computing,\\
Minnan Normal University, Zhangzhou 363000, China}

%%%%%%%%%%%%%%%%%%%%%%%%%%%%%%%%%%%%%%%%%%%%%%%%%%%%%%%%%%%%%%%%%%%%%%%%%%%%%%%%%%%%%%%%%%%%%%%%%%%%%%%%%%%%%%%%%%%%%%%%%%%%%%
% Enter your name between curly braces

%\date{\today}  % Enter your date or \today between curly braces

\date{\today}          % Enter your date or \today between curly braces
\maketitle
%\large

\begin{abstract}
The introduction of covering-based rough sets has made a substantial contribution to
the classical rough sets. However, many vital problems in rough sets, including attribution reduction, are NP-hard and
therefore the algorithms for solving them are usually greedy. Matroid, as a generalization of linear
independence in vector spaces, it has a variety of applications in many fields such as algorithm design and combinatorial
optimization. An excellent introduction to the topic of rough matroids is due to Zhu and Wang. On the basis of their work, we study
the rough matroids based on coverings in this paper. First, we investigate some properties of the definable sets with respect to a covering.
Specifically, it is interesting that the set of all definable sets with respect to a covering, equipped with the binary relation of inclusion $\subseteq$, constructs a lattice.
Second, we propose the rough matroids based on coverings, which are a generalization of the rough matroids based on relations.
Finally, some properties of rough matroids based on coverings are explored. Moreover, an equivalent formulation of rough matroids based on coverings is presented. These interesting and important results exhibit many potential connections between rough sets and matroids.\\                                                                                                                                                                                                                %%%%%%%%%%%%%%%%%%%%%%%%%%%%%%%%%%%%


\end{abstract}

%%%%%%%%%%%%%%%%%%%%%%%%%%%%%%%%%%%%

\section{Introduction}
Rough sets were originally proposed by Pawlak~\cite{Pawlak82Rough,Pawlak91Rough,Pawlak02Rough}
as useful tools for dealing with the vagueness and granularity in information systems. This theory can approximately characterize
an arbitrary subset of a universe by using two definable subsets
called lower and upper approximation operators~\cite{ChenZhangYeungTsang06Rough}. Now, with
the fast development of rough sets in recent years, it has
already been applied in fields such as knowledge discovery~\cite{GalvezDCG00AnApplication,JohnsonLiuChen00Unification,LeungWuZhang06Knowledge},
machine learning~\cite{HanHuCercone00Supervised,Kosko92Neural}, decision analysis~\cite{GrecoMatarazzoSlowinski00Rough}, process control~\cite{Inuiguchi04Generalizations,Kryszkiewicz98Rough}, pattern
recognition~\cite{LeeKCHCoMine03Efficient,MinLiuFang08Rough} and many other areas~\cite{BazanOsmolskiSkowronSlezakSzczukaWroblewski02Rough,BittnerStell00Rough,ChenLi12AnApplication,DejaSlezak01Rough,LashinKozaeKhadraMedhat05Rough,LiZhangMa05Rough}.
Covering-based rough sets have been proposed as a
generalization of classical rough sets and the study on covering-based rough sets is fetching in more and more researchers in the past few years.

The concept of matroids was originally introduced by Whitney~\cite{H.WhitneyOn} in 1935 as a
generalization of graph theory and linear algebra. Matroid is a structure that
generalizes linear independence in vector spaces, and has a variety of applications in
many fields such as algorithm design~\cite{Edmonds71Matroids} and combinatorial optimization~\cite{Lawler01Combinatorialoptimization}. In theory,
matroid provides a well platform to connect it with other theories.
Some interesting results about the connection between matroids and rough sets can be
found in literatures~\cite{HuangZhu12geometric,LiZhu13Covering,LiLiu12Matroidal,LiuZhuZhang12Relationshipbetween,LiuZhu12matroidal,TangSheZhu12matroidal,WangZhuZhuMin12Matroidal,WangZhu12Quantitative,WangZhu13Equivalent,YangZhu13Matroidal,YangZhu13covering-basedrough}.

This paper proposes the concept of rough matroids based on coverings, as an excellent generalization of the rough matroids based on relations, is due to Zhu and Wang~\cite{Zhuwang13roughmatroids} for integrating rough sets and matroids. With this new concept, we can not only study rough sets with matroidal structures, but can also investigate matroids from a wider perspective. Begin to this paper, we define the definable set based on a covering, which generalizes from the definable set based on a relation. Moreover, on the basis of the definition of definable set based on a covering, some properties of definable set based on a covering are investigated. Specifically, it is interesting that the set of all definable sets with respect to a covering, equipped with the binary relation of inclusion $\subseteq$, constructs a lattice. Second, we propose the concept of rough matroids based on coverings from the viewpoint of definable set based on a covering. In fact, the definition of rough matroids based on coverings is a generalization of rough matroids based on relations. Finally, some properties of rough matroids based on coverings are explored. Moreover, an equivalent formulation of rough matroids based on coverings are presented.
These results show many potential connections between covering-based rough sets and matroids.

The remainder of this paper is organized as follows: In Section~\ref{section1}, some basic concepts and properties related to covering-based rough sets, lattices, matroids and rough matroids based on relations are introduced. In Section~\ref{section2}, we propose the rough matroids based on coverings, which are a generalization of the rough matroids based on relations. In Section~\ref{section3}, some properties of rough matroids based on coverings are explored. Section~\ref{section4} concludes this paper.
%%%%%%%%%%%%%%%%%%%%%%%%%%%%%%%%%%%%%%%%%%%%%%%%%%%%%%%%%%%%%%%%%%%%%%%%%%%%%%%
%%%%%%%%%%%%%%%%%%%%%%%%%%%%%%%%%%%%%%%%%%%%%%%%%%%%%%%%%%%%%%%%%%%%%%%%%%%%%%%
\section{Preliminaries}
\label{section1}
In this section, we recall some fundamental concepts and properties of covering-based rough sets, lattice and rough matroids based on relations. We firstly present several denotations.

Let $U$ be a set, $U\times U$ the product set of $U$ and $U$. Any subset $R$ of $U\times U$ is
called a binary relation on $U$. For any $(x$, $y)\in U\times U$, if $(x$, $y)\in R$, then we say $x$
has relation with $y$, and denote this relationship as $xRy$.

For any $x\in U$, we call the set $\{y\in U\mid xRy\}$ the successor neighborhood of $x$ in $R$ and denote
it as $RN(x)$.

Throughout this paper, a binary relation is simply called a relation and it is defined on a finite and nonempty set.

\subsection{Covering-based rough sets}
In this subsection, we review some basic definitions and results of covering-based rough sets used in this paper. For detailed descriptions about covering-based rough sets, please refer to
~\cite{Yao98Constructive,YaoYao12Covering,ZhuWang03Reduction}.
\begin{definition}(Covering~\cite{ZhuWang03Reduction})
Let $U$ be a universe of discourse, $\mathbf{C}$ a family of subsets of $U$. If none subsets in $\mathbf{C}$
is empty, and $\bigcup \mathbf{C}=U$, then $\mathbf{C}$ is called a covering of $U$.
\end{definition}

It is clear that a partition of $U$ is certainly a covering of $U$, so the concept of a covering is an extension of the concept of a partition. In the following discussion, the universe of discourse $U$ is considered to be finite.

\begin{definition}(Covering-based approximation space~\cite{ZhuWang03Reduction})
Let $U$ be a universe of discourse and $\mathbf{C}$ be a covering of $U$.
We call the ordered pair $\langle U, \mathbf{C}\rangle$ a covering-based approximation space.
\end{definition}

\begin{definition}(Neighborhood~\cite{ZhuWang07Topological})
Let $\mathbf{C}$ be a covering of $U$. For any $x\in U$, we define the neighborhood of $x$ as follows:
\begin{center}
$N_{\mathbf{C}}(x)=\bigcap\{K\in \mathbf{C}: x\in K\}$.
\end{center}
\end{definition}

\begin{definition}(Lower and upper approximations~\cite{ZhuWang07Topological})
$\forall X\subseteq U$, the lower approximation of $X$ is defined as
\begin{center}
$X_{+}=\{x\in U: N_{\mathbf{C}}(x)\subseteq X\}$
\end{center}
and the upper approximation of $X$ is defined as
\begin{center}
$X^{+}=\{x\in U: N_{\mathbf{C}}(x)\bigcap X\neq\emptyset\}$.
\end{center}

Operations $XL_{\mathbf{C}}$ and $XH_{\mathbf{C}}$ on $2^{U}$ defined as follows:
\begin{center}
$XL_{\mathbf{C}}(X)=X_{+}, XH_{\mathbf{C}}(X)=X^{+}$
\end{center}
are called the lower approximation operator and the upper approximation operator, coupled with the covering $\mathbf{C}$, respectively. When the covering is clear, we omit the lowercase $\mathbf{C}$ for the two operations.
\end{definition}

Let $\sim X=U-X$, we have the following property of the approximation operations, which defined in the above definition.
\begin{theorem}\cite{ZhuWang07Topological}
Let $\mathbf{C}$ be a covering of $U$ and $X\subseteq U$. Then
\begin{center}
$XL(\sim X)=\sim XH(X)$.
\end{center}
\end{theorem}

\subsection{Matroids}
Matroid is a vital structure with high applicability and borrows extensively from linear algebra and graph theory. It has been
applied to a variety of fields such as combinatorial optimization~\cite{Lawler01Combinatorialoptimization} and greedy algorithm design~\cite{Edmonds71Matroids}. One of main characteristic of matroids is that there are many equivalent ways to define them, which
is the basis for its powerful axiomatic system. The following definition presents a widely used axiomatization on matroids.

\begin{definition}(Matroid~\cite{Lai01Matroid,LiuChen94Matroid})
\label{matroids}
A matroid is an ordered pair $M=(U, \mathcal{I})$, , where $U$ is a finite set, and $\mathcal{I}$  a family of subsets of $U$ with the following three properties:\\
(I1) $\emptyset\in \mathcal{I}$;\\
(I2) If $I\in \mathcal{I}$, and $I'\subseteq I$, then $I'\in\mathcal{I}$;\\
(I3) If $I_{1}$, $I_{2}\in \mathcal{I}$, and $|I_{1}|<|I_{2}|$, then there exists $e\in I_{2}-I_{1}$ such that $I_{1}\bigcup\{e\}\in\mathcal{I}$, where $|I|$ denotes the cardinality of $I$.\\
Any element of $\mathcal{I}$ is called an independent set.
\end{definition}

\begin{example}
Let $G=(V$, $U)$ be the graph as shown in Fig.\ref{fig1}. Denote $\mathcal{I}=\{I\subseteq U\mid I$ does not contain
a cycle of $G\}$, i.e., $\mathcal{I}=\{\emptyset$, $\{a_{1}\}$, $\{a_{2}\}$, $\{a_{3}\}$, $\{a_{4}\}$, $\{a_{1}$, $a_{2}\}$, $\{a_{1}$, $a_{3}\}$, $\{a_{1}$, $a_{4}\}$, $\{a_{2}$, $a_{3}\}$, $\{a_{2}$, $a_{4}\}$, $\{a_{3}$, $a_{4}\}$, $\{a_{1}$, $a_{2}$, $a_{4}\}$, $\{a_{1}$, $a_{3}$, $a_{4}\}$, $\{a_{2}$, $a_{3}$, $a_{4}\}\}$. Then $M=(U$, $\mathcal{I})$ is a matroid, where $U=\{a_{1}$, $a_{2}$, $a_{3}$, $a_{4}\}$.
\end{example}

\begin{figure}
\begin{center}
  \includegraphics[width=3cm]{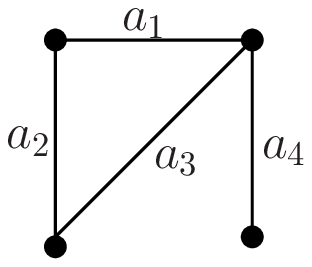}\\
  \caption{A graph}\label{fig1}
\end{center}
\end{figure}

%\begin{definition}(Rank function~\cite{Lai01Matroid})
%Let $M=(U$, $\mathcal{I})$ be a matroid. We can define the rank function of $M$ as follows: for all $X\in 2^{U}$
%\begin{center}
%$r_{M}(X)=max\{|I|\mid I\subseteq X$, $I\in\mathcal{I}\}$.
%\end{center}
%We call $r_{M}(X)$ the rank of $X$ in $M$.
%\end{definition}
%
%A matroid uniquely determines its rank function, and vice versa. The following theorem indicates that a matroid can be defined from
%the viewpoint of rank function.
%\begin{theorem}(Rank function axiom~\cite{Lai01Matroid})
%Let $r: 2^{U}\rightarrow \mathbf{N}$ be a function. Then there exists $M=(U, \mathcal{I})$ such that $r=r_{M}$ if and only if $r$ satisfies the following three conditions:
%\begin{flushleft}
%(R1) $0\leq r(X)\leq|X|$ for all $X\in 2^{U}$;\\
%(R2) If $X\subseteq Y\subseteq U$, then $r(X)\leq r(Y)$;\\
%(R3) If $X, Y\subseteq U$, then $r(X)+r(Y)\geq r(X\bigcup Y)+r(X\bigcap Y)$.
%\end{flushleft}
%\end{theorem}

\subsection{Partially ordered set and lattice}

Lattice with both order structures and algebraic structures, and it is closely linked
with many disciplines, such as group theory~\cite{AjmalJain09Someconstructions} and so on. Lattice theory plays an important role in many disciplines of computer science and engineering. For example, they have applications in distributed computing, programming language semantics~\cite{MedinaOjeda-Aciego10Multi-adjoint,Medina12Multi-adjoint,SyrrisPetridis11Alattice}.
\begin{definition}(Partially ordered set~\cite{Birhoff95Lattice,EstajiHooshmandaslDavva12Rough})
Let $U$ be a nonempty set and $\leq$ a partial order on $U$. For any $x, y, z\in U$, if
\begin{flushleft}
(1) $x\leq x$;\\
(2) $x\leq y$ and $y\leq x$ imply $x=y$;\\
(3) $x\leq y$ and $y\leq z$ imply $x\leq z$;
\end{flushleft}
then $\langle U, \leq\rangle$ is called a partially ordered set.
\end{definition}

Based on the partially ordered set, we introduce the concept of lattice.
\begin{definition}(Lattice~\cite{Birhoff95Lattice,EstajiHooshmandaslDavva12Rough})
A partially ordered set $\langle U, \leq\rangle$ is a lattice if $a\vee b$ and $a\wedge b$
exist for all $a, b\in U$. $\langle U, \vee, \wedge\rangle$ is called an algebraic system induced by lattice $\langle U, \leq\rangle$.
\end{definition}

We list the properties of the algebraic system induced by a lattice in the following theorem.

\begin{theorem}(~\cite{Birhoff95Lattice,EstajiHooshmandaslDavva12Rough})
\label{THLattice}
Let $\langle U, \vee, \wedge\rangle$ be an algebraic system induced by lattice $\langle U, \leq\rangle$.
For any $a, b, c\in U$, the algebra has the following identities:
\begin{flushleft}
(P1) $a\vee a=a, a\wedge a=a$;\\
(P2) $a\vee b=b\vee a, a\wedge b=b\wedge a$;\\
(P3) $(a\vee b)\vee c=a\vee (b\vee c)$;\\
(P4) $a\vee (a\wedge b)=a, a\wedge (a\vee b)=a$.
\end{flushleft}
\end{theorem}

The following definition shows the conditions that an algebraic system is a lattice.
\begin{definition}(~\cite{Birhoff95Lattice,EstajiHooshmandaslDavva12Rough})
\label{lattice8}
Let $\langle U, \vee, \wedge\rangle$ be an algebraic system. If $\wedge$ and $\vee$ satisfy (P2)-(P4) of Theorem~\ref{THLattice}, then
$\langle U, \vee, \wedge\rangle$ is a lattice.
\end{definition}
\subsection{Rough matroids based on relations}
Rough matroids based on relations have been proposed by Zhu and Wang~\cite{Zhuwang13roughmatroids} as a generalization of matroids.
That is, matroids are a special case of rough matroids based on relations. Rough matroids based on relations not only reflect
some important characteristics of rough sets, but also have the advantages of matroids.
\begin{definition}(Definable set based on a relation~\cite{Zhuwang13roughmatroids})
Let $R$ be a relation on $U$. For all $X\subseteq U$, if $X=\underset{x\in X}{\bigcup}RN(x)$, then $X$
is called a definable set with respect to $R$. The family of all definable sets with respect to $R$ is denoted
by $\mathbb{D}(U, R)$.
\end{definition}

In the following definition, we define lower rough matroids based on relations using the lower approximation operator.
\begin{definition}(Lower rough matroid based on a relation~\cite{Zhuwang13roughmatroids})
Let $R$ be a relation on $U$. A lower rough matroid based on $R$ is an ordered pair
$\underline{M}_{R}=(U, \underline{\mathcal{I}}_{R})$ where $\underline{\mathcal{I}}_{R}\subseteq\mathbb{D}(U, R)$
satisfies the following three conditions:

\begin{flushleft}
(LI1) $\emptyset\in\underline{\mathcal{I}}_{R}$;\\
(LI2) If $I\in \underline{\mathcal{I}}_{R}$, $I'\in\mathbb{D}(U, R)$ and $\underline{R}(I')\subseteq \underline{R}(I)$, then $I'\in\underline{\mathcal{I}}_{R}$;\\
(LI3) If $I_{1}, I_{2}\in\underline{\mathcal{I}}_{R}$ and $|\underline{R}(I_{1})|<|\underline{R}(I_{2})|$, then there exists $I\in\underline{\mathcal{I}}_{R}$ such that $\underline{R}(I_{1})\subset \underline{R}(I)\subseteq \underline{R}(I_{1})\bigcup \underline{R}(I_{2})$.
\end{flushleft}
\end{definition}

Similarly, the definition of upper rough matroids based on relations is presented using the upper approximation operator.
\begin{definition}(Upper rough matroid based on a relation~\cite{Zhuwang13roughmatroids})
Let $R$ be a relation on $U$. An upper rough matroid based on $R$ is an ordered pair
$\overline{M}_{R}=(U, \overline{\mathcal{I}}_{R})$ where $\overline{\mathcal{I}}_{R}\subseteq\mathbb{D}(U, R)$
satisfies the following three conditions:

\begin{flushleft}
(UI1) $\emptyset\in\overline{\mathcal{I}}_{R}$;\\
(UI2) If $I\in \overline{\mathcal{I}}_{R}$, $I'\in\mathbb{D}(U, R)$ and $\overline{R}(I')\subseteq \overline{R}(I)$, then $I'\in\overline{\mathcal{I}}_{R}$;\\
(UI3) If $I_{1}, I_{2}\in\overline{\mathcal{I}}_{R}$ and $|\overline{R}(I_{1})|<|\overline{R}(I_{2})|$, then there exists $I\in\overline{\mathcal{I}}_{R}$ such that $\overline{R}(I_{1})\subset \overline{R}(I)\subseteq \overline{R}(I_{1})\bigcup \overline{R}(I_{2})$.
\end{flushleft}
\end{definition}

The following example is used to illustrate the lower and upper rough matroids based on relations.
\begin{example}
Let $U=\{a_{1}, a_{2}, a_{3}, a_{4}\}, R_{1}=\{(a_{1}, a_{1}), (a_{2}, a_{1}), (a_{2}, a_{2}), (a_{3}, a_{1}), (a_{3},$\\$ a_{3})\}$
and $R_{2}=R_{1}\bigcup\{(a_{4}, a_{4})\}$. Suppose that $\mathcal{I}=\{\emptyset, \{a_{1}\}, \{a_{1}, a_{2}\}, \{a_{1}, a_{3}\}\}$. Then
$M=(U, \mathcal{I})$ is an upper rough matroids based on $R_{1}$ and a lower rough matroid based on $R_{2}$.
\end{example}
\section{Rough matroids based on coverings}
\label{section2}
Rough set is a useful tool for dealing with the vagueness and granularity in information systems. It is widely used in attribute reduction
in data mining. There are many optimization issues in attribute reduction. Matroid theory is a branch of combinatorial mathematics. It is widely used in optimization. Therefore, it is a good idea to integrate rough sets and matroids. In theory, matroid theory provides a good platform to connect it with rough set theory. An excellent introduction to
the topic of rough matroids is due to Zhu and Wang. In this section, on the basis of their work, we propose rough matroid based on a covering, which is a generalization of rough matroid based on a relation.
\subsection{Definable sets based on coverings}
%%%%%%%%%%%%%%%%%%%%%%%%%%%%%%%%%%%%%%%%%%%%%%%%%%%%%%%%%%%%%%%%%%%%%%%%%%%%%%%%%%%%%%%%%%%%%%%%%%%%%%%%%%%%%%%%%%%%%%%%%%%%%%%%%%%%%%
We firstly present the definition of definable set based on a covering, which is the foundation of rough matroids based on coverings.
\begin{definition}(Definable set based on a covering)
\label{definition1}
Let $\mathbf{C}$ be a covering of $U$. For all $X\subseteq U$, if $X=\underset{x\in X}{\bigcup}N_{\mathbf{C}}(x)$, then $X$
is called a definable set with respect to $\mathbf{C}$. The family of all definable sets with respect to $\mathbf{C}$ is denoted
by $\mathbb{D}(U, \mathbf{C})$.
\end{definition}

%%%%%%%%%%%%%%%%%%%%%%%%%%%%%%%%%%%%%%%%%%%%%%%%%%%%%%%%%%%%%%%%%%%%%%%%%%%%%%%%%%%%%%%%%%%%%%%%%%%%%%%%%%%%%%%%%%%%%%%%%%%%%%%%%%%%%%
Note that $\emptyset\in\mathbb{D}(U, \mathbf{C})$ since $\emptyset=\underset{x\in\emptyset}{\bigcup}N_{\mathbf{C}}(x)$, which is essential for
completeness in some results of this paper.

In fact, the definable set as the above definition shown is a similar generalization of that in relation-based rough sets. We illustrate it with the following example.
%%%%%%%%%%%%%%%%%%%%%%%%%%%%%%%%%%%%%%%%%%%%%%%%%%%%%%%%%%%%%%%%%%%%%%%%%%%%%%%%%%%%%%%%%%%%%%%%%%%%%%%%%%%%%%%%%%%%%%%%%%%%%%%%%%%%%%
\begin{example}
\label{example1}
Let $U=\{a, b, c, d, e, f\}$ and $\mathbf{C}=\{\{e, f\}, \{a, d, e\}, \{a, d, f\}, \{b, c, e\},$\\$ \{b, c, f\}, \{a, b, c, d\}\}$. Then\\
\begin{center}
\begin{tabular}{cccccc}
   \hline
   $~~~~~~N_{\mathbf{C}}(a)$ & ~~~~~~$N_{\mathbf{C}}(b)$ & $~~~~~~N_{\mathbf{C}}(c)$ & $~~~~~~N_{\mathbf{C}}(d)$ & $~~~~~~N_{\mathbf{C}}(e)$ & $~~~~~~N_{\mathbf{C}}(f)$ \\
   \hline
   $~~~~~~\{a, d\}$ & $~~~~~~\{b, c\}$ & $~~~~~~\{b, c\}$ & $~~~~~~\{a, d\}$ & $~~~~~~\{e\}$ & $~~~~~~\{f\}$ \\
   \hline\\
 \end{tabular}
\end{center}
Let $X=\{b, d, f\}$ and $Y=\{a, b, c, d\}$. Then $X$ is not a definable set with respect to $\mathbf{C}$ since $X\neq\underset{x\in X}{\bigcup}N_{\mathbf{C}}(x)=N_{\mathbf{C}}(b)\bigcup N_{\mathbf{C}}(d)\bigcup N_{\mathbf{C}}(f)=\{a, b, c, d, f\}$. Conversely, $Y$
is a definable set with respect to $\mathbf{C}$ since $Y=\underset{y\in Y}{\bigcup}N_{\mathbf{C}}(y)=N_{\mathbf{C}}(a)\bigcup N_{\mathbf{C}}(b)\bigcup N_{\mathbf{C}}(c)\bigcup$\\$ N_{\mathbf{C}}(d)=\{a, b, c, d\}$. The family of all definable sets with respect to $\mathbf{C}$ is
$\mathbb{D}(U, \mathbf{C})=\{\emptyset, \{e\}, \{f\}, \{a, d\}, \{b, c\}, \{e, f\}, \{a, d, e\}, \{a, d, f\}, \{b, c, e\}, \{b, c, f\}, \{a, b, c, d\},$\\$\{a, d, e, f\}, \{b, c, e, f\}, \{a, b, c, d, e\}, \{a, b, c, d, f\}, \{a, b, c, d, e, f\}\}$.
\end{example}

\begin{proposition}
\label{proposition1}
If $X, Y\in\mathbb{D}(U, \mathbf{C})$, then $X\bigcap Y, X\bigcup Y\in\mathbb{D}(U, \mathbf{C})$.
\end{proposition}
\begin{proof}
(1) According to Definition~\ref{definition1}, we need to prove only $X\bigcap Y=\underset{z\in X\bigcap Y}{\bigcup}N_{\mathbf{C}}(z)$.
On one hand, for any $z\in X\bigcap Y$, since $z\in N_{\mathbf{C}}(z)$, then $z\in\underset{z\in X\bigcap Y}{\bigcup}N_{\mathbf{C}}(z)$. Hence, $X\bigcap Y\subseteq\underset{z\in X\bigcap Y}{\bigcup}N_{\mathbf{C}}(z)$. On the other hand, if $z'\in\underset{z\in X\bigcap Y}{\bigcup}N_{\mathbf{C}}(z)$, then there exists $z\in X\bigcap Y$ such that $z'\in N_{\mathbf{C}}(z)$. Since $X, Y\in\mathbb{D}(U, \mathbf{C})$,
then $z'\in N_{\mathbf{C}}(z)\subseteq X$ and $z'\in N_{\mathbf{C}}(z)\subseteq Y$. Therefore, $\underset{z\in X\bigcap Y}{\bigcup}N_{\mathbf{C}}(z)\subseteq X\bigcap Y$. Then $X\bigcap Y\in\mathbb{D}(U, \mathbf{C})$.\\
(2) Since $X, Y\in\mathbb{D}(U, \mathbf{C})$, then
\begin{center}
$X\bigcup Y=(\underset{x\in X}{\bigcup}N_{\mathbf{C}}(x))\bigcup(\underset{y\in Y}{\bigcup}N_{\mathbf{C}}(y))=\underset{z\in X\bigcup Y}{\bigcup}N_{\mathbf{C}}(z)$.
\end{center}
Hence, $X\bigcup Y\in\mathbb{D}(U, \mathbf{C})$.
\end{proof}

The above proposition shows that the union and intersection of two definable sets with respect to $\mathbf{C}$ is also definable set with respect
to $\mathbf{C}$, respectively. What interesting us is that the set of all definable sets with respect to $\mathbf{C}$, equipped with the binary relation of inclusion $\subseteq$, constructs a lattice.

\begin{example}
\label{example2}
Let $U=\{a, b, c\}$ and $\mathbf{C}=\{\{a, b\}, \{b, c\}\}$. Then \\
\begin{center}
$N_{\mathbf{C}}(a)=\{a, b\}, N_{\mathbf{C}}(b)=\{b\}$ and $N_{\mathbf{C}}(c)=\{b, c\}$.
\end{center}

Therefore, $\mathbb{D}(U, \mathbf{C})=\{\emptyset, \{b\}, \{a, b\}, \{b, c\}, \{a, b, c\}\}$. It is easy to prove the partially ordered set $\langle\mathbb{D}(U, \mathbf{C}), \subseteq\rangle$ is a lattice. We draw a picture as Figure~\ref{f1} shown.

\begin{figure}
  \begin{center}
  % Requires \usepackage{graphicx}
  \includegraphics[width=3.5cm]{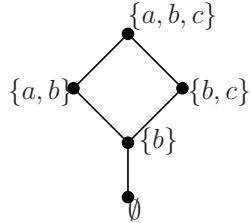}\\
  \caption{A lattice induced by $\mathbf{C}$}\label{f1}
  \end{center}
\end{figure}
\end{example}

 Similarly, the lattice $\langle\mathbb{D}(U, \mathbf{C}), \subseteq\rangle$ in Example~\ref{example1} is drawn as shown in Figure~\ref{f2}.
\begin{figure}
 \begin{center}
  % Requires \usepackage{graphicx}
  \includegraphics[width=10cm]{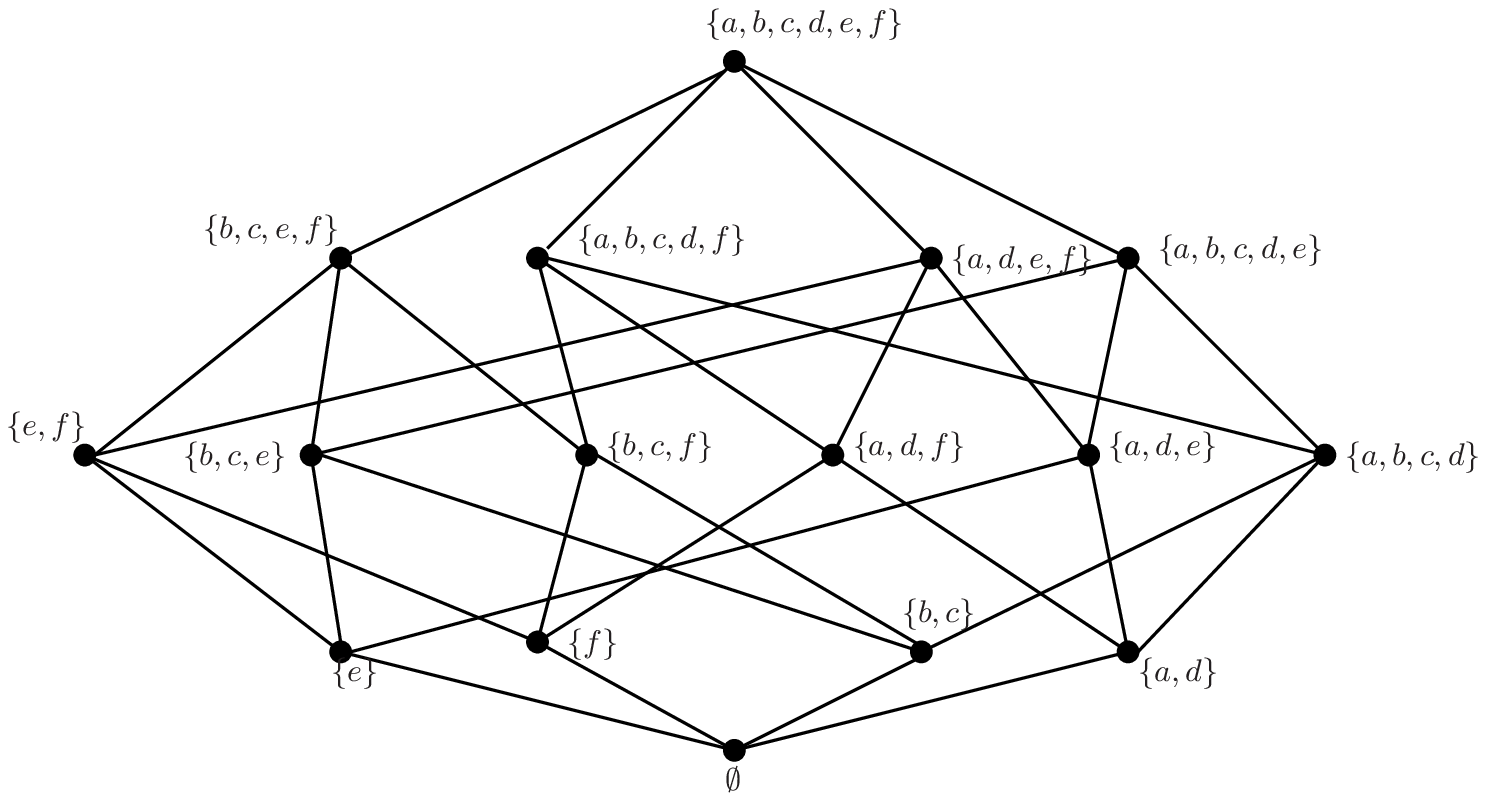}\\
  \caption{A lattice induced by $\mathbf{C}$}\label{f2}
 \end{center}
\end{figure}

\begin{proposition}
\label{proposition2}
Let $\mathbf{C}$ be a covering of $U$. Then the partially ordered set $\langle\mathbb{D}(U, \mathbf{C}), \subseteq\rangle$ is a lattice.
\end{proposition}

\begin{proof}
Suppose $\vee=\bigcup$ and $\wedge=\bigcap$. It is easy to prove this proposition by Definition~\ref{lattice8} and Proposition~\ref{proposition1}.
\end{proof}

In fact, the partially ordered set $\langle\mathbb{D}(U, \mathbf{C}), \subseteq\rangle$ is not only
a lattice, but also it is an atomic lattice.

\begin{proposition}
\label{proposition3}
Let $\mathbf{C}$ be a covering of $U$. Then
\begin{center}
$\mathbb{D}(U, \mathbf{C})=\{X\subseteq U: XL_{\mathbf{C}}(X)=X\}$.
\end{center}
\end{proposition}

\begin{proof}
For all $D\in\mathbb{D}(U, \mathbf{C})$, that is, $D=\underset{d\in D}{\bigcup}N_{\mathbf{C}}(d)$, therefore $D\subseteq XL_{\mathbf{C}}(D)$.
Since $x\in N_{\mathbf{C}}(x)$ for all $x\in U$, then $d\notin XL_{\mathbf{C}}(D)$ for all $d\notin D$. Hence, $XL_{\mathbf{C}}(D)\subseteq D$. This proves that $ XL_{\mathbf{C}}(D)=D$.
Therefore, $\mathbb{D}(U, \mathbf{C})\subseteq\{X\subseteq U: XL_{\mathbf{C}}(X)=X\}$. Conversely, for all $X\in \{X\subseteq U: XL_{\mathbf{C}}(X)=X\}$, that is $XL_{\mathbf{C}}(X)=X$, then $N_{\mathbf{C}}(x)\subseteq X$ for all $x\in X$. Hence,
$X=\underset{x\in X}{\bigcup}\{x\}\subseteq \underset{x\in X}{\bigcup}N_{\mathbf{C}}(x)\subseteq X$. This proves that
$\{X\subseteq U: XL_{\mathbf{C}}(X)=X\}\subseteq\mathbb{D}(U, \mathbf{C})$, completing the proof.
\end{proof}

\begin{corollary}
\label{corollary1}
Let $\mathbf{C}$ be a covering of $U$. Then
\begin{center}
$\mathbb{D}(U, \mathbf{C})=\{X\subseteq U: XH_{\mathbf{C}}(X)=X\}$.
\end{center}
\end{corollary}

\begin{proof}
Since $XL(\sim X)=\sim XH(X)$ holds for any $X\subseteq U$, then $XL_{\mathbf{C}}(X)=X\Leftrightarrow XH_{\mathbf{C}}(X)=X$.
This completes the proof.
\end{proof}

The following proposition shows an important property of the definable sets based on coverings, which plays a vital
role in the rough matroids based on coverings.

\begin{proposition}
Let $\mathbf{C}$ be a covering of $U$. Suppose that $D_{1}, D_{2}\in \mathbb{D}(U, \mathbf{C})$, $|D_{1}|<|D_{2}|$
and $d\in D_{2}-D_{1}$. Then $D_{1}\bigcup\{d\}\notin \mathbb{D}(U, \mathbf{C})$ if and only if there exists $e\in D_{2}-D_{1}$ such that
$d\neq e$ and $e\in N_{\mathbf{C}}(d)$.
\end{proposition}

\begin{proof}
$\Rightarrow)$: If $\forall e\in D_{2}-D_{1}((d=e)\vee (e\notin N_{\mathbf{C}}(d))$, then $N_{\mathbf{C}}(d)=\{d\}$. Therefore,
$D_{1}\bigcup\{d\}=D_{1}\bigcup N_{\mathbf{C}}(d)=(\underset{d_{1}\in D_{1}}{\bigcup}N_{\mathbf{C}}(d_{1})\bigcup N_{\mathbf{C}}(d))=\underset{d_{2}\in D_{1}\bigcup\{d\}}{\bigcup}N_{\mathbf{C}}(d_{2})$, which contradicts $D_{1}\bigcup\{d\}\notin \mathbb{D}(U, \mathbf{C})$.

$\Leftarrow)$: If there exists $e\in D_{2}-D_{1}$ such that
$d\neq e$ and $e\in N_{\mathbf{C}}(d)$. Hence $D_{1}\bigcup\{d\}\notin\mathbb{D}(U, \mathbf{C})$.
\end{proof}
\subsection{Rough matroids based on coverings}
In the following definition, we define lower rough matroids based on coverings using the covering-based lower approximation operator.
\begin{definition}(Lower rough matroid based on a covering)
Let $\mathbf{C}$ be a covering of $U$. A lower rough matroid based on $\mathbf{C}$ is an ordered pair
$\underline{M}_{\mathbf{C}}=(U, \underline{\mathcal{I}}_{\mathbf{C}})$ where $\underline{\mathcal{I}}_{\mathbf{C}}\subseteq\mathbb{D}(U, \mathbf{C})$
satisfies the following three conditions:

\begin{flushleft}
(LI1) $\emptyset\in\underline{\mathcal{I}}_{\mathbf{C}}$;\\
(LI2) If $I\in \underline{\mathcal{I}}_{\mathbf{C}}$, $I'\in\mathbb{D}(U, \mathbf{C})$ and $XL_{\mathbf{C}}(I')\subseteq XL_{\mathbf{C}}(I)$, then $I'\in\underline{\mathcal{I}}_{\mathbf{C}}$;\\
(LI3) If $I_{1}, I_{2}\in\underline{\mathcal{I}}_{\mathbf{C}}$ and $|XL_{\mathbf{C}}(I_{1})|<|XL_{\mathbf{C}}(I_{2})|$, then there exists $I\in\underline{\mathcal{I}}_{\mathbf{C}}$ such that $XL_{\mathbf{C}}(I_{1})\subset XL_{\mathbf{C}}(I)\subseteq XL_{\mathbf{C}}(I_{1})\bigcup XL_{\mathbf{C}}(I_{2})$.
\end{flushleft}
\end{definition}

Similarly, the definition of upper rough matroids based on coverings is presented using the covering-based upper approximation operator.

\begin{definition}(Upper rough matroid based on a covering)
\label{definition2}
Let $\mathbf{C}$ be a covering of $U$. An upper rough matroid based on $\mathbf{C}$ is an ordered pair
$\overline{M}_{\mathbf{C}}=(U, \overline{\mathcal{I}}_{\mathbf{C}})$ where $\overline{\mathcal{I}}_{\mathbf{C}}\subseteq\mathbb{D}(U, \mathbf{C})$
satisfies the following three conditions:

\begin{flushleft}
(UI1) $\emptyset\in\overline{\mathcal{I}}_{\mathbf{C}}$;\\
(UI2) If $I\in \overline{\mathcal{I}}_{\mathbf{C}}$, $I'\in\mathbb{D}(U, \mathbf{C})$ and $XH_{\mathbf{C}}(I')\subseteq XH_{\mathbf{C}}(I)$, then $I'\in\overline{\mathcal{I}}_{\mathbf{C}}$;\\
(UI3) If $I_{1}, I_{2}\in\overline{\mathcal{I}}_{\mathbf{C}}$ and $|XH_{\mathbf{C}}(I_{1})|<|XH_{\mathbf{C}}(I_{2})|$, then there exists $I\in\overline{\mathcal{I}}_{\mathbf{C}}$ such that $XH_{\mathbf{C}}(I_{1})\subset XH_{\mathbf{C}}(I)\subseteq XH_{\mathbf{C}} (I_{1})\bigcup XH_{\mathbf{C}}(I_{2})$.
\end{flushleft}
\end{definition}

The following example is used to illustrate the connection between lower and upper rough matroids based on coverings.

\begin{example}
\label{example3}
Let $U=\{a, b, c, d, e, f\}$ and $\mathbf{C}=\{\{e, f\}, \{a, d, e\}, \{a, d, f\}, \{b, c, e\},$\\$ \{b, c, f\}, \{a, b, c, d\}\}$. As shown
in Example~\ref{example1}, $\mathbb{D}(U, \mathbf{C})=\{\emptyset, \{e\}, \{f\}, \{a, d\}, \{b, c\}, $\\$\{e, f\}, \{a, d, e\}, \{a, d, f\}, \{b, c, e\}, \{b, c, f\}, \{a, b, c, d\}, \{a, d, e, f\}, \{b, c, e, f\}, \{a, b, $\\$c, d, e\}, \{a, b, c, d, f\}, \{a, b, c, d, e, f\}\}$.

(1) Suppose $\mathcal{I}_{\mathbf{C}}=\{\emptyset, \{e\}, \{f\}, \{a, d\}, \{a, d, e\}, \{a, d, f\}\}$. Then $M_{\mathbf{C}}=(U, \mathcal{I}_{\mathbf{C}})$ is both a lower and an upper rough matroid based on $\mathbf{C}$.

(2) Suppose $\mathcal{I}_{\mathbf{C}}=\{\emptyset, \{e\}, \{a, d\}, \{b, c\},\{a, d, e\}, \{a, b, c, d\}\}$. Then $M_{\mathbf{C}}=(U,$\\$ \mathcal{I}_{\mathbf{C}})$ is not a lower rough matroid based on $\mathbf{C}$ or an upper rough matroid based on $\mathbf{C}$ since $|XL_{\mathbf{C}}(\{e\})|=1<|XL_{\mathbf{C}}(\{a, d\})|=2$ but there dose not exist a $I\in\mathcal{I}_{\mathbf{C}}$ such that $XL_{\mathbf{C}}(\{e\})\subset XL_{\mathbf{C}}(I)\subseteq XL_{\mathbf{C}}(\{e\})\bigcup XL_{\mathbf{C}}(\{a, d\})$
and $|XH_{\mathbf{C}}(\{e\})|=1<|XH_{\mathbf{C}}(\{a, d\})|=2$ but there dose not exist a $I\in\mathcal{I}$ such that $XH_{\mathbf{C}}(\{e\})\subset XH_{\mathbf{C}}(I)\subseteq XH_{\mathbf{C}}(\{e\})\bigcup XH_{\mathbf{C}}(\{a, d\})$.
\end{example}

It is noted that for any covering $\mathbf{C}$, $(U, \mathbb{D}(U, \mathbf{C}))$ and $(U, \emptyset)$ are both lower and upper rough matroids based on $\mathbf{C}$.
\begin{definition}(Rough matroid based on a covering)
\label{7777}
Let $\mathbf{C}$ be a covering of $U$. A rough matroid based on $\mathbf{C}$ is an ordered pair
$M_{\mathbf{C}}=(U, \mathcal{I}_{\mathbf{C}})$ where $\mathcal{I}_{\mathbf{C}}\subseteq\mathbb{D}(U, \mathbf{C})$
satisfies the following three conditions:

\begin{flushleft}
(CI1) $\emptyset\in\mathcal{I}_{\mathbf{C}}$;\\
(CI2) If $I\in \mathcal{I}_{\mathbf{C}}$, $I'\subseteq I$ and $I'\in\mathbb{D}(U, \mathbf{C})$, then $I'\in\mathcal{I}_{\mathbf{C}}$;\\
(CI3) If $I_{1}, I_{2}\in\mathcal{I}_{\mathbf{C}}$ and $|I_{1}|<|I_{2}|$, then there exists $I\in\mathcal{I}_{\mathbf{C}}$ such that $I_{1}\subset I\subseteq I_{1}\bigcup I_{2}$.
\end{flushleft}
\end{definition}

\begin{example}
\label{666666}
Let $\mathbf{C}=\{\{a, b\}, \{a, c\}, \{a, b, c\}, \{c, d\}\}$ be a covering of $U=\{a, b, c, d\}$. Then $N_{\mathbf{C}}(a)=\{a\}$, $N_{\mathbf{C}}(b)=\{a, b\}$, $N_{\mathbf{C}}(c)=\{c\}$, $N_{\mathbf{C}}(d)=\{c, d\}$. Therefore, $\mathbb{D}(U, \mathbf{C})=\{\emptyset, \{a\}, \{c\}, \{a, b\}, \{a, c\}, \{c, d\}, \{a, b, c\}, \{a, c, d\}, \{a, b, c, d\}\}$.

(1) Suppose $\mathcal{I}_{\mathbf{C}}=\{\emptyset, \{a\}, \{c\}\}$. Then $M_{\mathbf{C}}=(U, \mathcal{I}_{\mathbf{C}})$ is a rough matroid based on $\mathbf{C}$.

(2) Suppose $\mathcal{I}_{\mathbf{C}}=\{\emptyset, \{a, b\}, \{a, c\}, \{a, c, d\}\}$. Then $M_{\mathbf{C}}=(U, \mathcal{I}_{\mathbf{C}})$ is not a rough matroid based on $\mathbf{C}$ since there does not exist $I\in \mathcal{I}_{\mathbf{C}}$ such that $\{a, b\}\subset I\subseteq\{a, b\}\bigcup\{a, c, d\}$.
\end{example}

It is clear that rough matroids based on coverings are generalizations of matroids. Condition under which the rough matroid based on a coverings is also a matroid are presented in the following proposition.
\begin{proposition}
Let $\mathbf{C}$ be a covering of $U$. Then $M_{\mathbf{C}}=(U, \mathcal{I}_{\mathbf{C}})$ is a matroid
if and only if $N_{\mathbf{C}}(x)=\{x\}$ for all $x\in \bigcup\mathcal{I}_{\mathbf{C}}$.
\end{proposition}
\begin{proof}
$\Rightarrow)$: If there exists $x\in \mathcal{I}_{\mathbf{C}}$ such that $\{x\}\subset N_{\mathbf{C}}(x)$, then $\{x\}\notin\mathbb{D}(U, \mathbf{C})$. This contradicts (I2) of Definition~\ref{matroids}.

$\Leftarrow)$: We only need to prove $\mathcal{I}_{\mathbf{C}}$ satisfies (I1)-(I3) of Definition~\ref{matroids}.
(I1) and (I2) are straightforward. Since $\mathcal{I}_{\mathbf{C}}$ satisfies (CI3), then for all $I_{1}, I_{2}\in\mathcal{I}_{\mathbf{C}}$ and $|I_{1}|<|I_{2}|$, there exists $I\in\mathcal{I}_{\mathbf{C}}$ such that $I_{1}\subset I\subseteq I_{1}\bigcup I_{2}$.

If $|I|=|I_{1}|+1$, then we suppose $e\in I-I_{1}$. Thus $I_{1}\bigcup\{e\}=\underset{x\in I_{1}}{\bigcup}N_{\mathbf{C}}(x)\bigcup \{e\}=\underset{x'\in I}{\bigcup}N_{\mathbf{C}}(x')=I\in\mathbb{D}(U, \mathbf{C})$, i.e., $I_{1}\bigcup\{e\}\in \mathcal{I}_{\mathbf{C}}$.

If $|I|\geq|I_{1}|+2$, then $I-I_{1}=\{x_{1}, x_{2}, \ldots, x_{m}\}(m\geq2)$. Since $I_{1}\subset I$, then $I_{1}\bigcup x_{i}=I-\{x_{1}, x_{2}, \ldots, x_{i-1}, x_{i+1}, \ldots, x_{m}\}(i\in\{1, 2, \ldots, m\})\in\mathbb{D}(U, \mathbf{C})$. Thus $I_{1}\bigcup\{x_{i}\}\in \mathcal{I}_{\mathbf{C}}$.

This completes the proof.
\end{proof}

The above proposition shows the condition under which the rough matroid based on a coverings is a matroid. We illustrate it with the following example.
\begin{example}
Let $C=\{\{a, b, c\}, \{a, b, d\}, \{a, c, d\}, \{b, c\}, \{d\}\}$ be a covering of $U=\{a, b, c, d\}$.
Then $N_{\mathbf{C}}(a)=\{a\}$, $N_{\mathbf{C}}(b)=\{b\}$, $N_{\mathbf{C}}(c)=\{c\}, N_{\mathbf{C}}(d)=\{d\}$. Therefore, $\mathbb{D}(U, \mathbf{C})=2^{U}$. Suppose $\mathcal{I}_{\mathbf{C}}=\{\emptyset, \{a\}, \{c\}\}$. Then $M_{\mathbf{C}}=(U, \mathcal{I}_{\mathbf{C}})$ is both a rough matroid based on $\mathbf{C}$ and a matroid on $U$.
\end{example}
\section{Properties of rough matroids based on coverings}
\label{section3}
This section shows some properties and equivalent formulations of rough matroids based on coverings.

\begin{proposition}
Let $n>r>0$ be two integers and $\mathbf{C}$ be a covering of $n-$ set $U$. Suppose
\begin{center}
$\mathcal{I}_{\mathbf{C}}=\{I\in \mathbb{D}(U, \mathbf{C}): |I|\leq r\}$.
\end{center}
If $N_{\mathbf{C}}(x)=\{x\}$ for all $x\in U$, then $M_{\mathbf{C}}=(U, \mathcal{I}_{\mathbf{C}})$ is a rough matroid based on $\mathbf{C}$.
\end{proposition}
\begin{proof}
If $N_{\mathbf{C}}(x)=\{x\}$ for all $x\in U$, then $\mathbb{D}(U, \mathbf{C})=2^{U}$. Thus $M_{\mathbf{C}}=(U, \mathcal{I}_{\mathbf{C}})$ is a rough matroid based on $\mathbf{C}$.
\end{proof}

\begin{example}(Continued from Example~\ref{666666})
As shown in Example~\ref{666666}, $M_{\mathbf{C}}=(U, \mathcal{I}_{\mathbf{C}})$ is a rough matroid based on $\mathbf{C}$, Where $\mathcal{I}_{\mathbf{C}}=\{\emptyset, \{a\}, \{c\}\}$. But $N_{\mathbf{C}}(d)=\{c, d\}\neq\{d\}$.
\end{example}

The above example indicates that $N_{\mathbf{C}}(x)=\{x\}$ for all $x\in U$ is a sufficient condition and not a necessary condition of $M_{\mathbf{C}}=(U, \mathcal{I}_{\mathbf{C}})$ is a rough matroid based on $\mathbf{C}$.
\begin{corollary}
Let $n\geq r\geq 1$ be two integers and $\mathbf{C}$ be a covering of $n-$ set $U$. Suppose
\begin{center}
$\mathcal{I}_{\mathbf{C}}=\{I\in \mathbb{D}(U, \mathbf{C}): |I|\leq r\}$.
\end{center}
Then $M_{\mathbf{C}}=(U, \mathcal{I}_{\mathbf{C}})$ is a matroid if and only if $N_{\mathbf{C}}(x)=\{x\}$ for all $x\in U$.
\end{corollary}
\begin{proof}
$\Rightarrow)$: If there exists $x\in U$ such that $N_{\mathbf{C}}(x)\neq\{x\}$, then $\{x\}\notin\mathbb{D}(U, \mathbf{C})$. This contradicts
$M_{\mathbf{C}}=(U, \mathcal{I}_{\mathbf{C}})$ is a matroid.

$\Leftarrow)$: It is straightforward.
\end{proof}

The above proposition presented as a generalization of the uniform matroids.
\begin{proposition}
Let $M_{\mathbf{C}_{1}}=(U_{1}, \mathcal{I}_{\mathbf{C}_{1}})$ and $M_{\mathbf{C}_{2}}=(U_{2}, \mathcal{I}_{\mathbf{C}_{2}})$ be two rough matroids based on coverings, $U_{1}\bigcap U_{2}=\emptyset$ and $U=U_{1}\bigcup U_{2}$. Then $M_{\mathbf{C}_{1}}\bigoplus M_{\mathbf{C}_{2}}=(U, \mathcal{I})$ is a rough matroid based on $\mathbf{C}_{1}\biguplus\mathbf{C}_{2}$.
\end{proposition}

\begin{proof}
It is easy to prove $\mathbf{C}=\mathbf{C}_{1}\biguplus\mathbf{C}_{2}$ be a covering of $U$. Since $U_{1}\bigcap U_{2}=\emptyset$, then
$\mathbb{D}(U, \mathbf{C})=\mathbb{D}(U_{1}, \mathbf{C}_{1})\biguplus\mathbb{D}(U_{2}, \mathbf{C}_{2})$.
Therefore we only need to prove $\mathcal{I}$ satisfies (CI1)-(CI3) of Definition~\ref{7777}.

(1) Since $\emptyset\in\mathcal{I}_{\mathbf{C}_{1}}$ and $\emptyset\in\mathcal{I}_{\mathbf{C}_{2}}$, then $\emptyset\in\mathcal{I}$, i.e., (CI1) holds.

(2) If $I\in\mathcal{I}$, then there exist $I'\in\mathcal{I}_{\mathbf{C}_{1}}$ and
$I''\in\mathcal{I}_{\mathbf{C}_{2}}$ such that $I=I'\bigcup I''$. Suppose $I^{*}\subseteq I$ and $I^{*}\in\mathbb{D}(U, \mathbf{C})$. Then
$I^{*}=I^{*}\bigcap I=(I^{*}\bigcap I')\bigcup(I^{*}\bigcap I'')$. Therefore $I^{*}\bigcap I'\in\mathbb{D}(U_{1}, \mathbf{C}_{1}), I^{*}\bigcap I'\subseteq I'$ and $I^{*}\bigcap I''\in\mathbb{D}(U_{2}, \mathbf{C}_{2}), I^{*}\bigcap I''\subseteq I''$. Since $\mathcal{I}_{\mathbf{C}_{1}}$ and $\mathcal{I}_{\mathbf{C}_{2}}$ satisfy (CI2), then $\mathcal{I}$ satisfies (CI2).

(3) If $I', I''\in\mathcal{I}$, then there exist $I_{1}', I_{1}''\in\mathcal{I}_{\mathbf{C}_{1}}$ and $I_{2}', I_{2}''\in\mathcal{I}_{\mathbf{C}_{2}}$ such that $I'=I_{1}'\bigcup I_{2}', I''=I_{1}''\bigcup I_{2}''$. Suppose $|I'|<|I''|$. Then $|I_{1}'\bigcup I_{2}'|<|I_{1}''\bigcup I_{2}''|$. Since $U_{1}\bigcap U_{2}=\emptyset$, then $|I_{1}'|<|I_{1}''|$ or $|I_{2}'|<|I_{2}''|$.

(a) If $|I_{1}'|<|I_{1}''|$, then there exists $I_{1}\in\mathcal{I}_{\mathbf{C}_{1}}$ such that $I_{1}'\subset I_{1}\subseteq I_{1}'\bigcup I_{1}''$.
Therefore $I_{1}'\bigcup I_{2}'=I'\subset I_{1}\bigcup I_{2}'\subseteq I_{1}'\bigcup I_{1}''\bigcup I_{2}'\subseteq I'\bigcup I''$. Since $I_{1}\bigcup I_{2}'\in\mathcal{I}$, then $\mathcal{I}$ satisfies (CI3).

(b) If $|I_{2}'|<|I_{2}''|$, then there exists $I_{2}\in\mathcal{I}_{\mathbf{C}_{2}}$ such that $I_{2}'\subset I_{2}\subseteq I_{2}'\bigcup I_{2}''$.
Therefore $I_{2}'\bigcup I_{1}'=I'\subset I_{2}\bigcup I_{1}'\subseteq I_{2}'\bigcup I_{2}''\bigcup I_{1}'\subseteq I'\bigcup I''$. Since $I_{2}\bigcup I_{1}'\in\mathcal{I}$, then $\mathcal{I}$ satisfies (CI3).

This completes the proof.
\end{proof}

The above proposition shows that the direct sum of two rough matroids based coverings is also a rough matroid based on a covering.
We illustrate it with the following example.
\begin{example}
Let $\mathbf{C}_{1}=\{\{b, c\}, \{a, c\}\}$ be a covering of $U_{1}=\{a, b, c\}$ and $\mathbf{C}_{2}=\{\{d, e, f\}, \{d, e, g\}, \{d, f, g\}, \{e, f\}, \{g\}\}$ be a covering of $U_{2}=\{d, e, f, g\}$. Suppose $\mathcal{I}_{\mathbf{C}_{1}}=\{\emptyset, \{c\}, \{a, c\}\}$ and $\mathcal{I}_{\mathbf{C}_{2}}=\{\emptyset, \{d\}, \{e\}, \{f\}, \{d, e\}, \{e, f\}\}$. Then $M_{\mathbf{C}_{1}}=(U_{1}, \mathcal{I}_{\mathbf{C}_{1}})$ is a rough matroid based on $\mathbf{C}_{1}$ and $M_{\mathbf{C}_{2}}=(U_{2}, \mathcal{I}_{\mathbf{C}_{2}})$ is a rough matroid based on $\mathbf{C}_{1}$. Suppose
\begin{center}
$\mathcal{I}=\mathcal{I}_{\mathbf{C}_{1}}\biguplus\mathcal{I}_{\mathbf{C}_{2}}=\{\emptyset, \{c\}, \{d\}, \{e\}, \{f\}, \{a, c\}, \{c, d\}, \{c, e\}, \{c, f\},$\\$ \{d, e\}, \{e, f\}, \{a, c, d\}, \{a, c, e\}, \{a, c, f\}, \{c, d, e\}, \{c, e, f\}, \{a, c, d, e\}, \{a, c, e, f\}\}$.
\end{center}
Then $M_{\mathbf{C}_{1}}\bigoplus M_{\mathbf{C}_{2}}=(U, \mathcal{I})$ is a rough matroid based on $\mathbf{C}$, where $U=U_{1}\bigcup U_{2}$ and $\mathbf{C}=\mathbf{C}_{1}\biguplus\mathbf{C}_{2}$.
\end{example}

From the viewpoint of matroids, the direct sum of rough matroids based on coverings is considered. For this purpose, a denotation is presented.
\begin{definition}
Let $\mathcal{A}_{1}$ and $\mathcal{A}_{2}$ be two families of subsets of $U$. One can denote
\begin{center}
$\mathcal{A}_{1}\biguplus\mathcal{A}_{2}=\{X_{1}\bigcup X_{2}: X_{1}\in\mathcal{A}_{1}, X_{2}\in\mathcal{A}_{2}\}$.
\end{center}
\end{definition}

\begin{definition}(Direct sum of rough matroids based on coverings)
Let $M_{\mathbf{C}_{1}}=(U_{1}, \mathcal{I}_{\mathbf{C}_{1}})$ and $M_{\mathbf{C}_{2}}=(U_{2}, \mathcal{I}_{\mathbf{C}_{2}})$ be two rough matroids based on coverings, $U_{1}\bigcap U_{2}=\emptyset$ and $U=U_{1}\bigcup U_{2}$.
Suppose
\begin{center}
$M_{\mathbf{C}_{1}}\bigoplus M_{\mathbf{C}_{2}}=(U, \mathcal{I})$, where $\mathcal{I}=\mathcal{I}_{\mathbf{C}_{1}}\biguplus\mathcal{I}_{\mathbf{C}_{2}}$.
\end{center}
Then we say $M_{\mathbf{C}_{1}}\bigoplus M_{\mathbf{C}_{2}}$ is the direct sum of $M_{\mathbf{C}_{1}}$ and $M_{\mathbf{C}_{2}}$.
\end{definition}
\begin{proposition}
Let $M_{\mathbf{C}}=(U, \mathcal{I}_{\mathbf{C}})$ be a rough matroids based on $\mathbf{C}$ if and only if $\mathcal{I}_{\mathbf{C}}$ satisfies (CI1), (CI2) and
\begin{flushleft}
(CI3)$'$ for every $D\in\mathbb{D}(U, \mathbf{C})$, all the maximal subsets of $D$ which are contained in $\mathcal{I}_{\mathbf{C}}$ have the same
cardinality.
\end{flushleft}
\end{proposition}

\begin{proof}
If (CI3)$'$ is not true, then there exists $D'\in\mathbb{D}(U, \mathbf{C})$ such that $I_{1}, I_{2}\in\mathcal{I}_{\mathbf{C}}$ are maximal subsets of $D'$ with $|I_{1}|\neq|I_{2}|$. Without lost of generality, we assume $|I_{1}|<|I_{2}|$. By (CI3), there exists $I\in\mathcal{I}_{\mathbf{C}}$ such
that $I_{1}\subseteq I$, which contradicts the maximality of $I_{1}$.

Conversely, suppose $\mathcal{I}_{\mathbf{C}}$ satisfies (CI1), (CI2) and (CI3)$'$. By Definition~\ref{7777}, it is necessary to show that
$\mathcal{I}_{\mathbf{C}}$ satisfies (CI3). Let $I_{1}, I_{2}\in\mathcal{I}_{\mathbf{C}}$ with $|I_{1}|<|I_{2}|$, set $D=I_{1}\bigcup I_{2}$. By (CI3)$'$,  all the maximal subsets of $D$ which are contained in $\mathcal{I}_{\mathbf{C}}$ have the same
cardinality. Since $|I_{1}|<|I_{2}|$, we must have $I_{1}\bigcup\{e\}\in\mathcal{I}_{\mathbf{C}}$, $I_{1}\bigcup\{e\}\subseteq D$ and $e\notin I_{1}$; it is obvious $I_{1}\subset I_{1}\bigcup\{e\}\subseteq I_{1}\bigcup I_{2}$, namely, $\mathcal{I}_{\mathbf{C}}$ satisfies (CI3).

This completes the proof.
\end{proof}

The above proposition shows an equivalent formulation of rough matroids based on coverings.

\section{Conclusions}
\label{section4}
In this paper, we proposed a framework for combining covering-based rough sets and matroids.
First, we investigated some properties of the definable sets with respect to a covering.
Specifically, it is interesting that the set of all definable sets with respect to a covering, equipped with the relation of inclusion $\subseteq$, constructs
a lattice.
Second, we proposed the rough matroids based on coverings, which are a generalization of the rough matroids based on relations.
Finally, some properties of rough matroids based on coverings are explored. Further, some important concepts such as base, et al. from matroids could generalize to the rough matroids based on coverings. Moreover, the equivalent formulations of rough matroids based on coverings maybe an interesting research topic.

\section{Acknowledgments}
This work is in part supported by the National Science Foundation of China
under Grant Nos. 61170128, 61379049, and 61379089 and the Postgraduate Education Innovation Base for Computer Application Technology,  Signal and Information Processing of Fujian Province (No. [2008]114, High Education of Fujian).

%\bibliographystyle{splncs}
%\bibliography{E:/bib/Roughsets}

%

\end{document}